\documentclass[letterpaper, 10 pt, conference]{ieeeconf}   % Comment this line out if you need a4paper

\IEEEoverridecommandlockouts
\usepackage{hyperref}
\usepackage{graphicx}
\usepackage[utf8]{inputenc}
\usepackage{cite}
\usepackage{multicol}
\usepackage{mathtools, cuted}
\usepackage[cmintegrals]{newtxmath}
\usepackage{mathtools}

\usepackage{subcaption}
\usepackage{stackengine}
\usepackage{epsfig}
\usepackage{amsmath}
\usepackage{algorithm}
\usepackage{algpseudocode}

\newtheorem{theorem}{Theorem}
\usepackage{caption}
\usepackage[usenames, dvipsnames]{color}
\newtheorem{assumption}{Assumption}
\graphicspath{{figures/}}

\title{\LARGE \bf Quadrotor Formation Flying Resilient to Abrupt Vehicle\\
Failures via a Fluid Flow Navigation Function
}
\author{Matthew Romano$^{a,\dagger}$, Harshvardhan Uppaluru$^{b,\dagger}$,
Hossein Rastgoftar$^{b}$, Ella Atkins$^{a}$%
\thanks{$^{\dagger}$ Both authors contributed equally to this work.}%
\thanks{$^a$Authors are with the Robotics Department at the University of Michigan {\tt\small\{mmroma, ematkins\}@umich.edu}}%
\thanks{$^b$Authors are with the Aerospace and Mechanical Engineering department at the University of Arizona {\tt\small\{huppaluru, hrastgoftar\}@arizona.edu}}%
}
\date{January 2022}

\begin{document}

\maketitle

\begin{abstract}
This paper develops and experimentally evaluates a navigation function for quadrotor formation flight that is resilient to abrupt quadrotor failures and other obstacles. The navigation function is based on modeling  healthy quadrotors as particles in an ideal fluid flow. We provide three key contributions: (i) A Containment Exclusion Mode (CEM) safety theorem and proof which guarantees safety and formally specifies a minimum safe distance between quadrotors in formation, (ii) A real-time, computationally efficient CEM navigation algorithm, (iii) Simulation and experimental algorithm validation. Simulations were first performed with a team of six virtual quadrotors to demonstrate velocity tracking via dynamic slide speed, maintaining sufficient inter-agent distances, and operating in real-time. Flight tests with a team of two custom quadrotors were performed in an indoor motion capture flight facility,  successfully validating that the navigation algorithm can handle non-trivial bounded tracking errors while guaranteeing safety.  
\end{abstract}

\section{Introduction}

Quadrotors are becoming increasingly popular due to advances in sensing, actuation, and computational power. A quadrotor is cheap and highly maneuverable with a wide range of useful applications. A single quadrotor has limited processing and payload carriage capabilities. However, a group of cooperatively controlled quadrotors can achieve difficult tasks with considerable advantages in terms of resilience to failure, mission complexity, and scalability.

\begin{figure}[ht]
    \centering
    \includegraphics[width=\columnwidth,trim=0 0 0 0, clip]{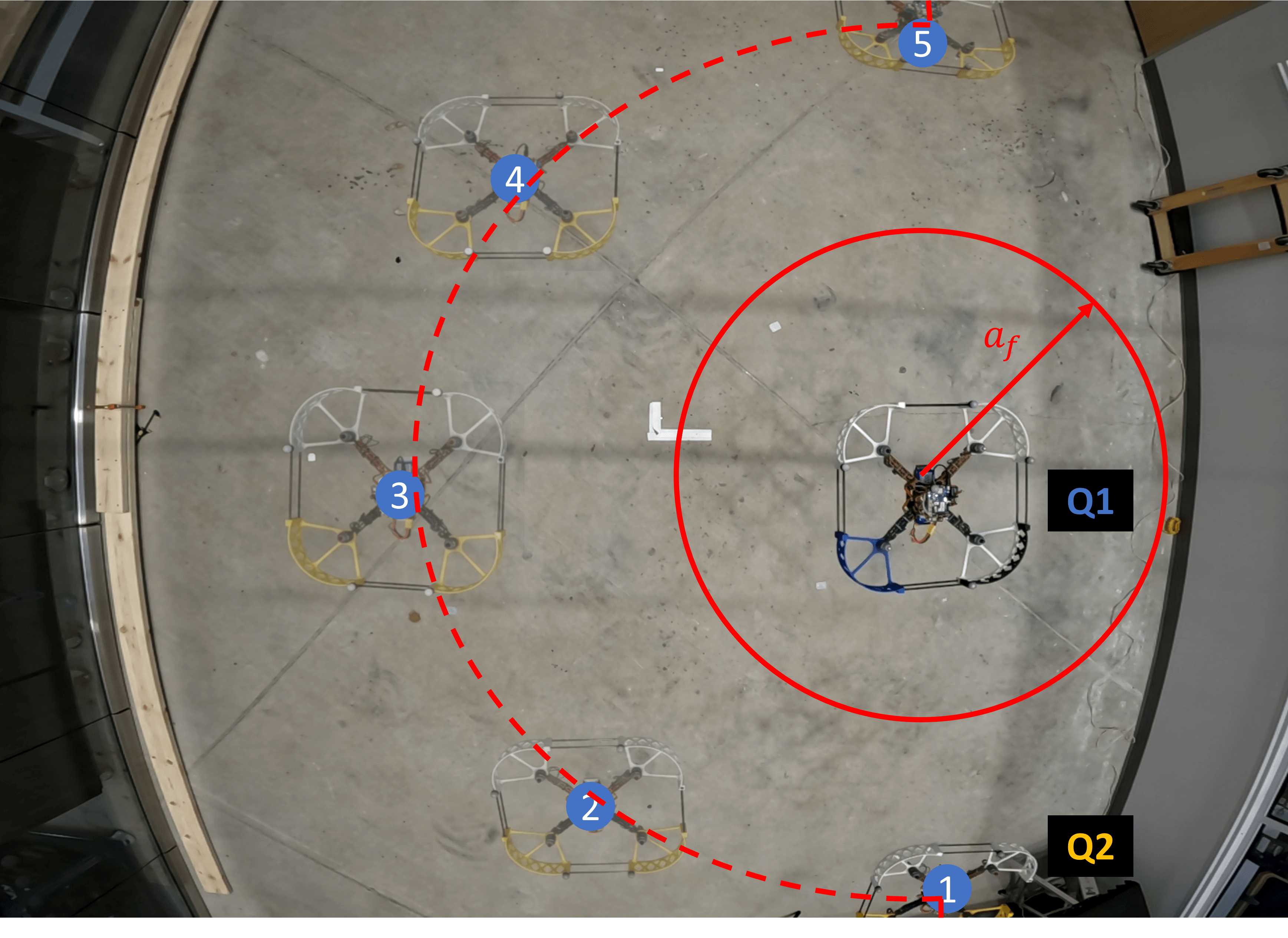}
    \caption{Q1 fails and stays in place within the solid red circle. Q2 uses the CEM navigation algorithm to smoothly avoid the failed vehicle and travels along the dashed red path. The blue circles indicate five time points for Q2.}
    \label{fig:teaser_figure}
\end{figure}

Several multi-agent coordination approaches have been investigated with applications ranging from surveillance \cite{du2017pursuing, da2017multi} to formation flying \cite{rojo2017implementation, ha2017multi}, rescue missions \cite{altair2017decision}, wildlife monitoring and exploration \cite{witczuk2018exploring}, precision agriculture \cite{tsouros2019review}, cooperative payload transport \cite{rastgoftar2018cooperative1, rossomando2020aerial} and hazardous environment sensing \cite{argrow2005uav}. Virtual structure (VS) is a centralized technique in which a group of agents, operating as particles of a virtual rigid body, preserve a strict geometric relationship to one another and a frame of reference %i.e., each agent in a multi-agent system acts as particles of a virtual rigid body 
\cite{lewis1997high, ren2004decentralized}. Desired agent trajectories are calculated in a centralized manner through rigid body rotations and translations of the virtual body in $3$-D motion space. However, due to the rigid body requirement and a single point of failure, the range of applications using VS is limited. %To tackle the disadvantage of a centralized approach, a decentralized virtual structure was also proposed \cite{ren2004decentralized}. 
Consensus is one of the most extensively studied cooperative control approaches in which a team of agents reach agreement/consensus by communicating only with their neighbours. It is a decentralized coordination approach and is broadly divided into two categories: consensus without a leader (i.e., leaderless consensus) \cite{rao2013sliding, ren2009distributed} and consensus with a leader (i.e., leader-follower consensus) \cite{song2010second, xu2016robust}. %Furthermore, multi-agent consensus under fixed communication topology \cite{wang2018fixed}, switching inter-agent communication \cite{wen2016group}, retarded consensus \cite{sun2020mean, ahmed2020consensus} and in the presence of communication delays \cite{zhou2018h, zhang2019delay} have been previously studied .
Containment control \cite{ferrari2006laplacian, li2013distributed, wang2013distributed, wen2015containment} is a decentralized leader-follower method where collective motion of all agents is achieved with multiple leaders. Defined by geometric constraints i.e., all agents are contained within a particular area, the follower agents obtain the desired positions through local communication with in-neighbor agents. %Single and double integrator dynamics were used for multi-agent system containment control modeling \cite{ wang2013distributed}. A containment control approach was presented in which the follower agents are modelled by linear dynamics \cite{li2013distributed}. 
Based on the principles of continuum mechanics, continuum deformation is a recent multi-agent coordination approach for which agents in a multi-agent system (MAS) are treated as particles of a body deforming under a homogeneous transformation \cite{rastgoftar2014evolution,rastgoftar2016continuum, rastgoftar2017multi, emadi2021physics}.

The idea of using potential fields for real-time path planning was originally proposed for robot arms \cite{khatib1985real}. An attractive potential is placed at the goal, repulsive potentials on obstacles, and gradient descent is used to plan a path. This can be called a navigation function. This is a useful approach, however it has issues of local minima. The randomized potential field approach attempts to get around this by executing a random walk whenever the planner gets stuck in a local minimum, but still provides no guarantee of planning completeness \cite{barraquand1991robot}. For harmonic functions subject to Laplace's equations, a navigation function can be constructed that only has saddle points as local minima such that they can be easily escaped \cite{connolly1990path}. This model has been applied to robot formations, but with explicit models of repulsive functions between robots, one goal sink node, and computation time on the order of one second. However, the scenario considered in this paper needs an algorithm to quickly respond to an abrupt failure on the order of milliseconds.

This paper proposes a novel safety recovery algorithm for a multi-quadrotor system (MQS) and experimentally validates the approach in simulation and quadrotor flight experiments. We consider each quadrotor in a MQS as a finite number of particles in an ideal fluid flow pattern. By classifying quadrotors as healthy and failed agents, we consider failed quadrotors as singularity points of the fluid flow.  We then define desired trajectories of the healthy agents along flow streamlines so that the failed quadrotors are safely partitioned outside the motion space (see Figure \ref{fig:teaser_figure}). We enclose the failed quadrotors with virtual obstacles to define no-fly zones. %pillars in indoor environments. 
Because we treat these obstacles as singularity points that are excluded from the motion space, our proposed safety-recovery approach is called \textit{containment exclusion mode (CEM)} throughout this paper.   

Compared to the existing literature, this paper offers the following contributions:
\begin{enumerate}
    \item We propose a real-time, computationally efficient CEM navigation algorithm that dynamically modifies quadrotor sliding speed to maintain a desired maximum speed for all quadrotors. %Trajectories for each quadrotor are computed in real-time %is computed as it goes like a navigation function. Avoids saddle points with additional random term.
    \item We formally specify safety and collision avoidance by providing two CEM safety theorems. Theorem \ref{thm1} applies to a large quadrotor system with multiple %failures/
    obstacles whereas Theorem \ref{thm:matt} has been developed for experiments for a single obstacle. % in a flow field during CEM
    \item We experimentally evaluate the CEM navigation algorithm using a team of quadrotors in formation while validating our CEM safety theorem. % i.e., this works, but there are probably a couple more things we could model including communication delay, computation time, more things about dynamics and relating that to the flow field gradients which velocities are derived from.
\end{enumerate}

% \subsection{Organization of the paper}
This paper is organized as follows.  A problem statement (Section \ref{Problem Statement}) is followed by a description of our proposed CEM approach (Section \ref{Sec:ProblemStatement}).
% in detail with a subsection explaining our CEM Planner Algorithm. 
 Section \ref{Sec:ExpSetup} presents our experimental setup, followed by experimental results in Section \ref{Sec:ExpResults}. Section \ref{Sec:Conclusion} concludes the paper.

\section{Problem Statement}
\label{Problem Statement}
% {\color{blue}We consider safety recovery in the }
Consider a system of quadrotors, flying in formation, defined by set $\mathcal{S} = \left\{1,\cdots,N_q\right\}$, where $N_q$ is the number of quadrotors. Let $N_f$ be the number of quadrotors that fail to follow the desired formation and are assumed to hover in this work. The failed quadrotors are identified by set $\mathcal{F} \subset \mathcal{S}$ and $N_f < N_q$. %The quadrotors in set $\mathcal{F}$ are part of set $\mathcal{O}$, defined as obstacle set, that can also include other obstacles such as no-fly zones, geo-fenced areas. 
The remaining quadrotors are termed as healthy quadrotors because they follow the desired formation, and is defined by the set $\mathcal{H}=\mathcal{S} \setminus \mathcal{F}$. 

At reference time $t_0$, $N_f$ quadrotors fail to follow the desired formation. The failed quadrotors are enclosed by virtual obstacles as defined below in Assumption \ref{Assumption0}. The healthy quadrotors should now follow the safety-recovery approach called CEM where the quadrotors in $\mathcal{H}$ should avoid the %failed quadrotors
failed quadrotors and safely pass them before continuing with their individual trajectories. 

Given the problem setup defined above, we first develop a navigation function where the healthy quadrotors are treated as particles in ideal fluid flow combining uniform flow and doublet flow in the $x-y$ plane. 

\section{Approach}
\label{Sec:ProblemStatement}

The set $\mathcal{F}$ satisfies the following assumption:
\begin{assumption}
    \label{Assumption0}
   We assume failed quadrotor is enclosed by a virtual obstacle defined as no-fly zone that is completely contained within an enclosed cylinder of radius $a_f$, centered at $z_f = (x_f, y_f)$, elongated in the $z$-direction (e.g. Figure \ref{fig:teaser_figure}).
\end{assumption}

This ideal fluid flow pattern is generated by the complex function
% If $z$ is a complex variable denoting position in $x-y$ plane such that $z = x + \mathbf{i}y$. Then $f(\mathbf{z})$ is a function of the complex variable and is defined as
\begin{equation}
    \centering
    f(\mathbf{z}) = \sum_{f \in \mathcal{F}} \left( \mathbf{z} - \mathbf{z}_f + \frac{a_p^2}{\mathbf{z} - \mathbf{z}_f} \right),
    \label{Eq:Functionz}
\end{equation}
where $\mathbf{z}=x+\mathbf{j}y$ is the complex variable, $\mathbf{z}_f =x_f+\mathbf{j}y_f$ is the position of singularity point $f\in \mathcal{F}$ with finite set $\mathcal{F}$ identifying all singularities in the $x-y$ plane, and $a_p=a_{f}+\delta+\epsilon$ is the planned exclusion radius. $a_{p}$ is larger than the actual exclusion radius ($a_{f}$) by a bound on healthy agent controller error ($\delta$) and healthy agent radius ($\epsilon$).
% , and the planned exclusion radius $a_{p}=a_{f}+\delta+\epsilon$ is used to account for healthy agent controller error ($\delta$) and radius ($\epsilon$). 
% Using complex analysis facilitates calculations to obtain the potential function $\phi(x, y)$ and stream function $\psi(x, y)$ of ideal fluid flow such that, 
Eq.  \eqref{Eq:Functionz} can be rewritten as
\begin{equation}
    \centering
    f(\mathbf{z}) = \phi(x, y) + \mathbf{i}\psi(x,y)
    \label{Eq:phiplusipsi}
\end{equation}
where 
\begin{subequations}
\begin{equation}
    \phi \left( x, y \right) 
    =  \sum_{f \in \mathcal{F}} 
    \left(
    \frac{\left(x-x_{f}\right)\left( \left(x-x_{f}\right)^2 + \left(y-y_{f}\right)^2 + a_p^2 \right)}{\left(x-x_{f}\right)^2 + \left(y-y_{f}\right)^2}
    \right)
    \label{eq:phi_xy}
\end{equation}
\begin{equation}
    \psi \left( x, y \right) 
    = \sum_{f \in \mathcal{F}} 
    \left(
    \frac{\left(y-y_{f}\right) \left( \left(x-x_{f}\right)^2 + \left(y-y_{f}\right)^2 -  a_p^2 \right) }{\left(x-x_{f}\right)^2 + \left(y-y_{f}\right)^2}
    \right)
    \label{eq:psi_xy}
\end{equation}
\end{subequations}
are called \textit{potential} and \textit{stream} functions, respectively. Note that Eq. \eqref{Eq:phiplusipsi} provides a conformal mapping between $x-y$ and $\phi-\psi$ planes, where Cauchy-Reimann and Laplace equations are satisfied by $\phi(x, y)$ and $\psi(x, y)$. 

\begin{equation*}
    \centering
    \frac{\partial \phi}{\partial x} = \frac{\partial \psi}{\partial y}
\end{equation*}
\begin{equation*}
    \centering
    \frac{\partial \phi}{\partial y} = - \frac{\partial \psi}{\partial x}
\end{equation*}
\begin{equation*}
    \centering
    \nabla^2\psi = 0, \hspace{3cm}\nabla^2\phi = 0
\end{equation*}

Using the ideal fluid flow model on healthy quadrotors $\mathcal{H}$ at any time $t \geq t_0$, $x$ and $y$ components are constrained to slide along the stream curve

\begin{equation}
    \centering
    \psi((x_i(t_0), y_i(t_0))) = \psi_{i,0}, \forall i \in \mathcal{H}
\end{equation}

\begin{assumption}
    \label{Assumption1}
    Velocity in the $z$-direction is $0$ i.e., the altitude remains constant for each healthy quadrotor recovery trajectory. 
\end{assumption}

\begin{assumption}
    \label{Assumption2}
    As the recovery approach is active, we assume that the failed quadrotors do not leave the enclosing no-fly cylinders of radius $a_f$. 
    % Therefore, $a_f$ is chosen sufficiently large or the healthy quadrotors move sufficiently fast.    
\end{assumption}

Using the approach defined above, the main goal of this paper is to present a real-time, safe and efficient CEM algorithm that can plan safe recovery trajectories for all healthy quadrotors. 
% By wrapping the failed quadrotors by a cylinder, the system of quadrotors can safely and quickly recover 
To this end, we first present CEM as a navigation algorithm in Section \ref{CEM Planner (Algorithm)}. Then, we formally specify safety for CEM by providing inter-agent collision avoidance conditions (Section \ref{subsec:safety_analysis}).

% Scenario & Approach
%Five vehicles are traveling in a formation using centralized control. At some point, ONLY? one vehicle "fails" and stops moving. The system recognizes this and wraps the failing vehicle with a potential field. The remaining vehicles move around the failed vehicle. As soon as all the healthy agents pass the failed vehicle, they continue towards their destination. 

\subsection{CEM Navigation Algorithm}
\label{CEM Planner (Algorithm)}

Algorithm \ref{alg:CEM_planner} guides agents along constant $\psi$ streamlines in the fluid flow pattern.
This is performed in real-time such that the "$t$ for loop" (line 12) is run every $\Delta T$ seconds a total of $m$ times. $\Delta T=10ms$ (100Hz) in this paper. To minimize computational overhead, this method acts as a navigation function, calculating the immediate command for each agent to avoid collisions and then incrementing the trajectory in $\phi$ on each subsequent loop.

% Algorithm / Software
\begin{algorithm}[ht]
    \caption{CEM Navigation}
    \label{alg:CEM_planner}
    \begin{algorithmic}[1] % The number tells where the line numbering should start
        \Procedure{CEM}{$\mathcal{F}$, $m$, $v_{des}$, $\mathbf{r}$}
            \State \textbf{Inputs:} Failed agents $\mathcal{F}$, number of timesteps $m$, desired speed $v_{des}$, healthy vehicle positions $\mathbf{r}$
            \State \textbf{Output:} Trajectory for each vehicle via sendCommands()
            \State
            \State $\mathbf{r}_{d} \gets$ $\mathbf{r}$
            \State $\Delta \phi \gets v_{des} * \Delta T$
            \For{i $\gets$ 1, $N_{h}$}
                \State $\phi[i]$ $\gets$ $\phi(\mathbf{r}[i],\mathcal{F})$ \Comment{Eq. \ref{eq:phi_xy}}
                \State $\psi[i]$ $\gets$ $\psi(\mathbf{r}[i],\mathcal{F})$ \Comment{Eq. \ref{eq:psi_xy}}
            \EndFor
            \State
            \For {t $\gets$ 1, m}
                % \State OneStepCEM()
                  \State $\phi_{last} \gets \phi$
                    \For{k $\gets$ 1, $K$}
                        % \State $\phi \gets \phi_{last}$
                        \For{i $\gets$ 1, $N_{h}$}
                                % \State Get position of $i$-th vehicle at failure time
                                \State $\phi[i] \gets \phi_{last}[i] + \Delta \phi$
                                % \State $\phi$ $\gets$ $\phi(x,y)$ (Eq. \ref{eq:phi_xy})
                                % \State $\psi$ $\gets$ $\psi(x,y)$ (Eq. \ref{eq:psi_xy})
                                \State $\mathbf{r}_{d}[i], \dot{\mathbf{r}}_{d}[i]$ $\gets$ CalcXY($\phi[i]$,$\psi[i]$,$\mathbf{r}_{d}[i]$)
                                % \State $\mathbf{r}_{d,i}(j) \gets x, y$
                        \EndFor
                        \State $v_{max} \gets \max_{i \in N_{h}}{||\dot{\mathbf{r}}_{d}[i]||}$
                        \State $\Delta \phi \gets \Delta \phi * v_{des} / v_{max}$ %\Comment{Dynamic $\Delta \phi$}
                    \EndFor
                    \State sendCommands($\mathbf{r}_{d}$,$\dot{\mathbf{r}}_{d}$)
                \State sleep($\Delta T$)
            \EndFor
        \EndProcedure
    \end{algorithmic}
\end{algorithm}

$\phi$ and $\psi$ values for each agent are initialized to their starting positions. Also, sliding speed $\Delta \phi$ is set to match the desired vehicle velocities for the non-distorted case. 
% This slide speed is adjusted on each subsequent iteration to regulate the maximum commanded speed of all agents to the desired speed.
At every time step, $\phi$ is increased by $\Delta \phi$ and the resulting position and velocity for each vehicle is calculated using "CalcXY" (Alg. \ref{alg:CalcXY}). Maximum velocity $v_{max}$ is then used to update $\Delta \phi$ for the next iteration to more closely match $v_{des}$ (Line 20). This update step assumes $v_{max}=C \Delta \phi_{last}$ and wants to calculate $\Delta \phi_{next}$ such that $v_{des}=C \Delta \phi_{next}$ where $C$ is a constant assumed to encapsulate the gradient for the small step sizes. Therefore, $\Delta \phi_{next} = \Delta \phi_{last} \frac{v_{des}}{v_{max}}$ should calculate the proper $\Delta \phi$ to obtain $v_{des}$. This is executed $K$ times. If the assumption that $C$ is constant is reasonable, then $K=2$ is a good selection which effectively plans twice per time step, first to calculate $\Delta \phi$ and second to calculate the trajectories using $\Delta \phi$. $K=2$ was used in this paper for the main results and varied for additional results to see its effects. After the $K$ loop,  commands are output for that time step and the process repeats for the next time step.

Alg. \ref{alg:CalcXY} calculates the position given $\psi$, $\phi$ using a gradient descent inspired approach. On each iteration ($n=20$ in this paper), the current estimate is used to calculate the associated $\phi$ and $\psi$ and then the estimate is updated in the direction of decreasing error via the inverse of the Jacobian matrix, defined as:

\begin{equation}
    \mathbf{J}
    (x,y,\mathcal{F})
    = 
    \begin{bmatrix}
        \frac{\partial \phi}{\partial x} & \frac{\partial \phi}{\partial y} \\
        \frac{\partial \psi}{\partial x} & \frac{\partial \psi}{\partial y}
    \end{bmatrix}.
    \label{eq:jacobian}
\end{equation}

\noindent  $\mathbf{r}_{noise}$ is added to the update step to escape saddle points on the edge of the exclusion zone as: 

\begin{equation}
    \label{eq:r_noise}
    \mathbf{r}_{noise}
    =
    \begin{cases}
        |\mathcal{N}(0,\sigma)|\hat{\mathbf{r}}_{a} +\mathcal{N}(0,\sigma) \hat{\mathbf{r}}_{b} & \text{if } || \mathbf{r} - \mathbf{r}_{f} || \leq a_{p}\\
        \mathbf{0}, & \text{else} 
    \end{cases}
\end{equation}

\noindent where $\mathcal{N}(0,\sigma)$ generates Gaussian noise with zero-mean and standard deviation $\sigma$, $\hat{\mathbf{r}}_{a}= \frac{\mathbf{r}-\mathbf{r}_{f}}{||\mathbf{r}-\mathbf{r}_{f}||}$ is the direction away from the exclusion zone, and $\hat{\mathbf{r}}_{b}= \begin{bmatrix}0 & 1\\ -1 & 0\end{bmatrix}\hat{\mathbf{r}}_{a}$ is perpendicular to $\hat{\mathbf{r}}_{a}$.
$\sigma=1$mm was used in this paper.
This procedure ensures that commanded positions never violate the planned exclusion radius $a_{p}$. Lastly, the first-order difference equation is used to calculate the velocity.

\begin{algorithm}[ht]
    \caption{CalcXY}
    \label{alg:CalcXY}
    \begin{algorithmic}[1] % The number tells where the line numbering should start
        \Procedure{CalcXY}{$\phi$,$\psi$,$\mathbf{r}$}
            \State $\mathbf{r}_{-1} \gets \mathbf{r}$
            \For{i $\gets$ 1, $n$}
                % \State Get position of $i$-th vehicle at failure time
                \State $\phi_{1}$ $\gets$ $\phi(\mathbf{r},\mathcal{F})$ \Comment{Eq. \ref{eq:phi_xy}}
                \State $\psi_{1}$ $\gets$ $\psi(\mathbf{r},\mathcal{F})$ \Comment{Eq. \ref{eq:psi_xy}}
                \State $J_{h}$ $\gets$ $J_{h}(\mathbf{r},\mathcal{F})$ \Comment{Eq. \ref{eq:jacobian}}
                \State $\mathbf{r} \gets \mathbf{r} - J_h^{-1} \begin{bmatrix} \phi_{1} - \phi \\ \psi_{1} - \psi \end{bmatrix} + \mathbf{r}_{noise}$ \Comment{Eq. \ref{eq:r_noise}}
            \EndFor
            \State $\dot{\mathbf{r}} \gets (\mathbf{r}-\mathbf{r}_{-1} )/ \Delta T$
            \State \Return $\mathbf{r}, \dot{\mathbf{r}}$
        \EndProcedure
    \end{algorithmic}
\end{algorithm}

\subsection{Safety Analysis}
\label{subsec:safety_analysis}

Collision avoidance between healthy agents and obstacles is assured via the construction of the ideal fluid flow pattern. The planned exclusion radius for obstacles is $\delta + \epsilon$ larger than the actual exclusion radius. Therefore, if no agents are initialized inside the planned exclusion radius, the navigation function will never take them closer than $\delta+\epsilon$ which guarantees safety. 

Inter-agent collision avoidance is more complex as it involves distortion in the $x-y$ plane along constant $\psi$ streamlines. Theorem \ref{thm1} provides a general, conservative, inter-agent collision avoidance condition while Theorem \ref{thm:matt} provides a tighter, more useful condition for the single obstacle case.

\begin{theorem}\label{thm1}
    Let $r_{i,0}$ be the position of healthy agent $i \in \mathcal{H}$ at the time of failure and is defined as 
    \begin{equation*}
    r_{i,0} = [x_{i, 0} , y_{i,0}]^{T}
    \end{equation*}
    Define $\phi_{i,0} = \phi(x_{i, 0}, y_{i,0})$, $\psi_{i,0} = \psi(x_{i, 0}, y_{i,0})$, 
    \begin{equation}\label{mindistphipsi}
        p_{min,0} = \stackunder{min}{\stackunder{i,j}{i$\neq$j}} \sqrt{(\phi_{i,0} - \phi_{j,0})^2 + (\psi_{i,0} - \psi_{j,0})^2},
    \end{equation}
    with $p_{min,0}$ the minimum separation distance between agents in the $\phi-\psi$ plane, and
       \begin{equation*}
        \lambda_{max} = \max_{x,y} \left(\left(\phi_x^2 + \phi_y^2\right)\right).
    \end{equation*}
Assume that the trajectory control error of each individual quadcopter does not exceed $\delta$ and every quadcopter can be enclosed by a ball of radius $\epsilon$.  Inter-agent collision avoidance during CEM mode is guaranteed, if the sliding speed  $\dot{\phi}_i=\dot{\phi}$ is the same for every healthy quadcopter, and 
% \begin{equation}\label{intcolcondition}
%     \frac{p_{min,0}^2}{\lambda_{max}}\geq 
% 4{\color{red}a_{p}^2}.
% \end{equation}
\begin{equation}\label{intcolcondition}
    \frac{p_{min,0}^2}{\lambda_{max}}\geq 
4( \delta + \epsilon)^2.
\end{equation}

\end{theorem}

\begin{proof}
Under conformal mapping \eqref{Eq:Functionz}, there exists a one-to-one mapping between  infinitesimal element $dx-dy$ in the $x-y$ plane and infinitesimal element $d\phi-d\psi$ in the $\phi-\psi$ plane, where $(dx,dy)$ and $(d\phi,d\psi)$ can be related by
% Due to conformal mapping between $\phi-\psi$ and $x-y$ plane, the relationship between these two planes can be written as below: 
\begin{equation}
    \begin{bmatrix}
        d\phi \\
        d\psi
    \end{bmatrix}
    = 
    \mathbf{J} \begin{bmatrix}
        dx \\
        dy
    \end{bmatrix}
    ,
\end{equation}
and $\mathbf{J}$ is the Jacobian matrix in \eqref{eq:jacobian}. Therefore,
\begin{equation}
            d\phi^2+ d\psi^2
    = 
    \begin{bmatrix}
        dx && dy
    \end{bmatrix}\mathbf{J}^T \mathbf{J} \begin{bmatrix}
        dx \\
        dy
    \end{bmatrix}
\end{equation}

 Using Cauchy-Reimann Relations, we can show that $\mathbf{J}^T\mathbf{J}=\mathrm{diag}\left(\phi_x^2 + \phi_y^2,\phi_x^2 + \phi_y^2\right)\in \mathbb{R}^{2\times 2}$ is diagonal, and thus, we can write

\begin{equation}\label{mainxyphipsi}
\begin{split}
     d\phi^2 + d\psi^2&
    = 
    \begin{bmatrix}
        dx&
        dy
    \end{bmatrix}
    \begin{bmatrix}
        \phi_x^2 + \phi_y^2 & 0 \\
        0 & \phi_x^2 + \phi_y^2
    \end{bmatrix}
    \begin{bmatrix}
        dx\\
        dy
    \end{bmatrix}
    \\
    =&\left(\phi_x^2 + \phi_y^2\right)\left(dx^2 + dy^2\right).
\end{split}
\end{equation}
This also implies that
\begin{equation}\label{ineqconsts}
    dx^2+dy^2\geq \frac{d\phi^2+d\psi^2}{\lambda_{max}},\qquad x\neq x_f,~y\neq y_f.
\end{equation}
Eq. \eqref{ineqconsts} implies that
\begin{equation}\label{dminnn}
    \left( d_{min} (t)\right)^2\geq \frac{\left(p_{min}(t)\right)^{2}}{\lambda_{max}},\qquad \forall t
\end{equation}
where
\begin{subequations}
\begin{equation}
    p_{min}(t) =  \stackunder{min}{\stackunder{i,j}{i$\neq$j}} \sqrt{{\left(\phi_{i}(t) - \phi_{j}(t)\right)}^2 +\left(\psi_{i}\left(t\right) - \psi_j\left(t\right)\right)^2} \qquad \forall t.
\end{equation}
\begin{equation}\label{dmint}
    d_{min}(t) =  \stackunder{min}{\stackunder{i,j}{i$\neq$j}} \sqrt{{\left(x_{i}(t) - x_{j}(t)\right)}^2 + \left(y_{i}\left(t\right) - y_j\left(t\right)\right)^2} \qquad \forall t.
\end{equation}
\end{subequations}
% and inter-agent collision avoidance is guaranteed.
When the sliding speed $\dot{\phi}_i=\dot{\phi}$ is the same for every healthy quadcopter $i$, the healthy agent team moves as particles of a rigid-body in the $\phi-\psi$ plane, and thus, inter-agent distances in the $\phi-\psi$ plane are time-invariant. As a result, the minimum separation distance of the desired formation in the $\phi-\psi$ plane can be assigned at reference time $t_0$, when the failed agent no-fly zone first appears. Therefore,
$
p_{min,0}=p_{min}(t)
$ and Eq. \eqref{dminnn} simplifies to 
\begin{equation}\label{dminnn}
    \left( d_{min} (t)\right)^2\geq \frac{\left(p_{min,0}\right)^{2}}{\lambda_{max}},\qquad \forall t
\end{equation}
% \begin{equation}\label{dminnn}
%     \left( d_{min} (t)\right)^2\geq \frac{\left(p_{min,0}\right)^{2}}{\lambda_{max}} \geq 4 ( \delta + \epsilon)^2,\qquad \forall t
% \end{equation}
Since $ d_{min} (t) \geq 2(\delta + \epsilon) \forall t$ is the collision avoidance condition, this implies that the inter-agent collision avoidance is assured, if Eq. \eqref{intcolcondition} is satisfied.

\end{proof}

Theorem \ref{thm1} provides a sufficient condition for CEM inter-agent collision avoidance without restricting the number of failures. However, condition \eqref{intcolcondition}  conservatively overestimates the required minimum separation distance and it does so in the $\phi-\psi$ plane instead of the $x-y$ plane which is harder to use in practice. To overcome this issue, Theorem \ref{thm:matt} provides a tighter bound on initial formation distances for a collision avoidance guarantee condition for CEM but it works when there exists a single failed agent in the $x-y$ plane.

\begin{theorem}\label{thm:matt}
Inter-agent collision between every agent is avoided if the minimum separation distance $d_{min,0}$ at reference time $t_0$, when the failed agent appears, satisfies the following condition:
\begin{equation}\label{mindist}
\resizebox{0.99\hsize}{!}{%
$
    d_{min,0}=\min\limits_{i,j,~i\neq j}\sqrt{\left(x_{i,0}-x_{j,0}\right)^2+\left(y_{i,0}-y_{j,0}\right)^2}\geq 2\left(\delta+\epsilon\right)+a_p.
$
}
\end{equation}

\end{theorem}

\begin{proof}
Inter-agent collision between agents is avoided if $d_{min}(t)$, defined by \eqref{dmint}, satisfies the following condition:
\[
    d_{min}(t) \geq 2 \left( \delta + \epsilon \right),\qquad  \forall t.
\]
When a single failure exists in the $x-y$ plane, $\psi(x,y)=0$ wraps the failed agent by a circle of radius $a_p$ with the center positioned at $\left(x_f,y_f\right)$ (see Eq. \eqref{eq:psi_xy}, Figure \ref{fig:phi_psi_contours}). The maximum contraction of the distance between two arbitrary recovery paths $\psi(x,y)=\psi_{i,0}$ and $\psi(x,y)=\psi_{j,0}$  is less than $a_p$, and it occurs at $x=x_f$. So, $d_{min}(t) \geq d_{min,0} - a_{p} =  2 \left( \epsilon + \delta \right) \forall t$. Therefore, inter-agent collision between every two agent pair is avoided if the minimum separation distance $d_{min,0}$, obtained on the configuration of the healthy agents at reference time $t_0$, satisfies inequality  \eqref{mindist}.

\begin{figure}
    \centering
    \includegraphics[width=\columnwidth,trim=0 0 0 0, clip]{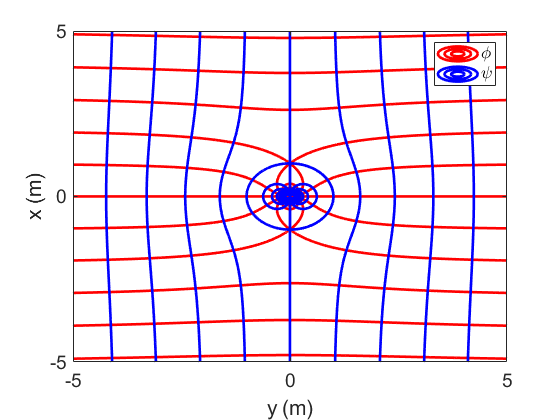}
    \caption{$\phi$, $\psi$ Contours. $a_{p}=1$, $x_{f}=y_{f}=0$. $\phi$ and $\psi$ integer levels are shown. The $\psi=0$ contour smoothly circumvents the exclusion zone.}
    \label{fig:phi_psi_contours}
\end{figure}

\end{proof}

\section{Experimental Setup}
\label{Sec:ExpSetup}

Figure \ref{fig:high_level_sw_block_diagram} shows a high level view of the experimental setup. The Ground Control Station (GCS) computer acts as the centralized planner for the system and sends desired states to each vehicle, where $\mathbf{X}_{d,i}=\begin{bmatrix} \mathbf{r}_{d,i}^{T} & \dot{\mathbf{r}}_{d,i}^{T} \end{bmatrix}^{T}$. An Intel core i7-9750H processor laptop was used as the GCS computer. The real-time commands were computed in C++ and sent over MAVLink using WiFi. The vehicles were equipped with motion capture markers and tracked via an indoor 3m x 4m, eight Optitrack camera motion capture space at 100Hz. The pose feedback, $\mathbf{X}_{i}=\begin{bmatrix} \mathbf{r}_{i}^{T} & \dot{\mathbf{r}}_{i}^{T} \end{bmatrix}^{T}$, was sent over low-latency 900MHz XBee wireless serial modules to the vehicles. M330 quadrotors running the rc\_pilot\_a2sys open source flight controller were used. A cascaded proportional-integral-derivative (PID) controller tracks the desired commands via motion capture feedback. The vehicles have an enclosing circle radius of $\epsilon=28cm$, and it is assumed $\delta=40cm$ bounds controller tracking error.

\begin{figure}
    \centering
    \includegraphics[width=0.5\columnwidth,trim=0 0 0 0, clip]{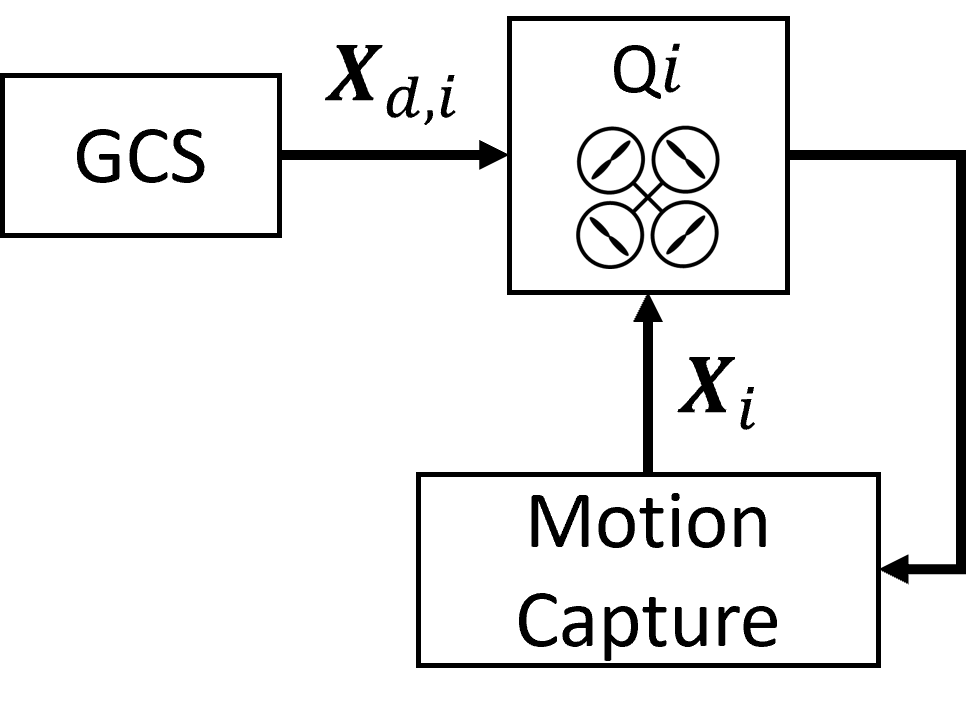}
    \caption{High Level Experimental Setup Block Diagram}
    \label{fig:high_level_sw_block_diagram}
\end{figure}

% \subsection{Formation \& CEM Planner Setup}
Figure \ref{fig:safety_radii} provides a visual depiction of the radii which keep the vehicles from colliding with each other per Sec. \ref{subsec:safety_analysis}. 
% These two terms were all that was needed for HDM mode. However, during CEM mode there is additional distortion that can happen and via the previous section we have determined that to be $a_{p}$. 
The minimum initial formation distance between any two vehicles is $d_{min}=2\left(\epsilon+\delta\right)+a_{p}$. Additionally, when a failure occurs and CEM mode is activated we assume that the failed agent will remain within $\delta$ of its setpoint so $a_{f}= \epsilon + \delta$.  
Therefore, $a_{p}=1.36m$ and $d_{min}=2.72m$.

\begin{figure}
    \centering
    \includegraphics[width=0.9\columnwidth,trim=0 0 0 0, clip]{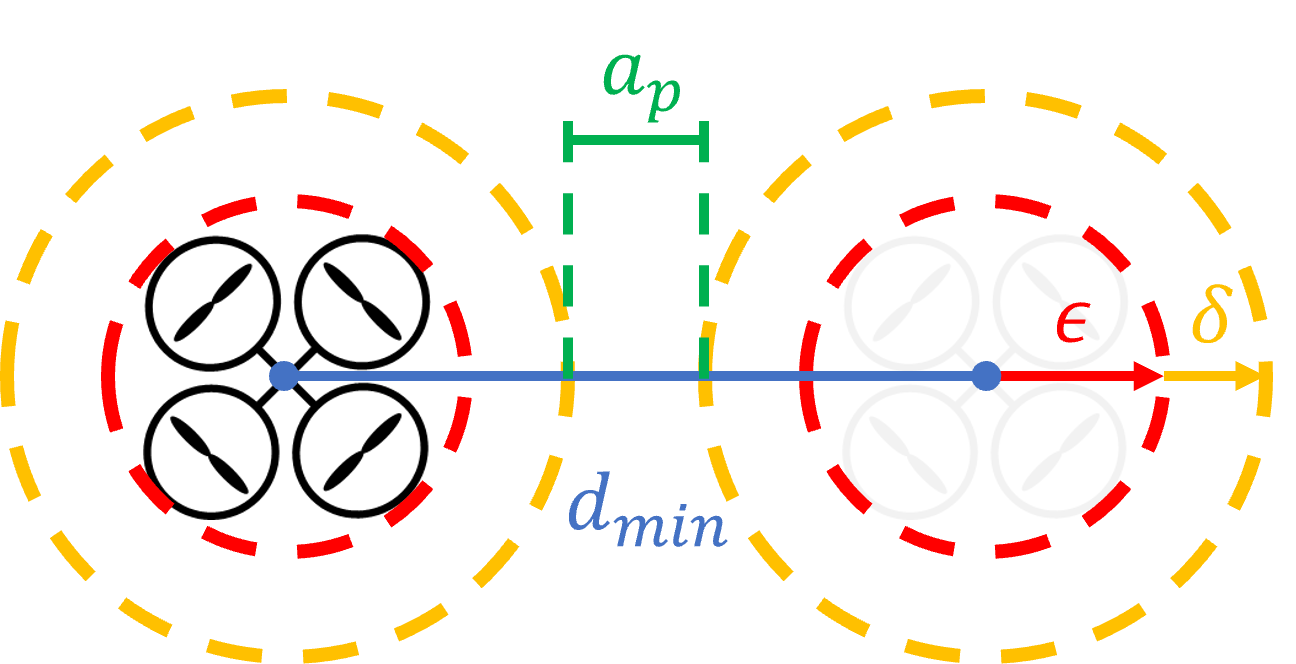}
    \caption{Formation inter-agent distances to guarantee safety.}
    \label{fig:safety_radii}
\end{figure}

\section{Experimental Results}
\label{Sec:ExpResults}
% Describe the test we ran, the formation, the trajectories

\subsection{CEM Navigation Validation with Six Virtual Vehicles}
For this test, the ground control station computer was configured to run exactly as described in the experimental setup section, except that only the setpoints were logged and no vehicles were flying. This was done to validate that the navigation algorithm was able to generate safe trajectories while also maintaining a desired velocity along constant streamlines and operating within runtime constraints.  

 Figure \ref{fig:six_vehicle_test_2d_plot} shows the formation and trajectory followed. The formation is similar to \cite{romano2019experimental} with a leading triangle and interior agents. The initial formation has agents with a minimum separation of 2.72m. The vehicles begin in a six agent formation on the bottom before Q1 fails. This activates the CEM navigation algorithm. An initial ($\phi$, $\psi$) coordinate is calculated for each vehicle and then on each iteration $\phi$ is advanced in a dynamic fashion to track a 1.0 m/s desired maximum velocity. Figure \ref{fig:six_vehicle_test_velocities} shows the commanded velocities. This shows each component of 2D velocity while the norm at each step was the quantity being tracked. Figure \ref{fig:six_vehicle_test_phi} shows the values of $\phi$ advancing while Figure \ref{fig:six_vehicle_test_dphi} shows how $\Delta \phi$ is changed to maintain the specified velocity. Notice how at about 6s and 13s $\Delta \phi$ gets very small. This coincides with Q2 traversing saddle points at each end of the exclusion zone. The Jacobian is large here which causes small changes in $\phi$ to produce a large change in the $x$-$y$ plane (i.e. Q2 moves fast while all other vehicles move slowly). Additionally, Figure \ref{fig:six_vehicle_test_nn} shows the distance between nearest agents. Every agent pair maintains at least $2(\delta + \epsilon)$ separation.

\begin{figure}
    \centering
    \includegraphics[width=\columnwidth,trim=0 0 0 0, clip]{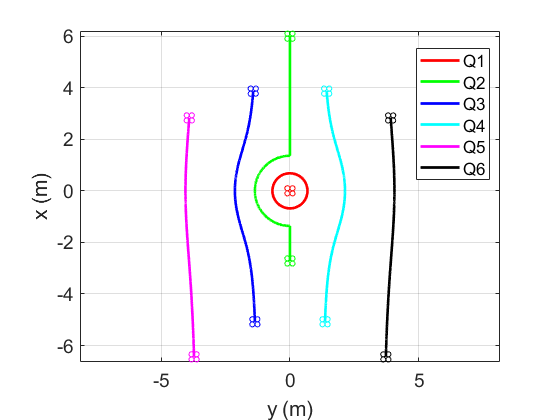}
    \caption{Six Vehicle Test 2D Plot}
    \label{fig:six_vehicle_test_2d_plot}
\end{figure}

\begin{figure}
    \centering
    \includegraphics[width=\columnwidth,trim=0 0 0 0, clip]{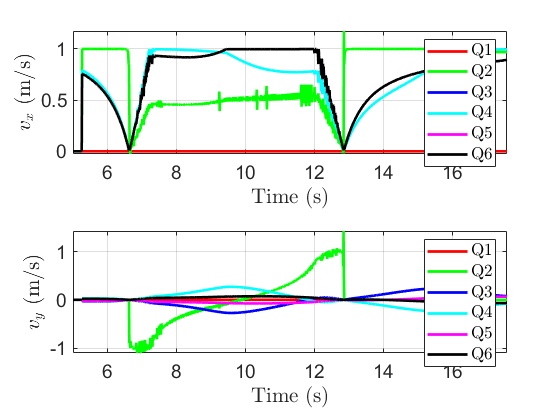}
    \caption{Six Vehicle Test Velocities}
    \label{fig:six_vehicle_test_velocities}
\end{figure}

\begin{figure}
    \centering
    \includegraphics[width=\columnwidth,trim=0 0 0 0, clip]{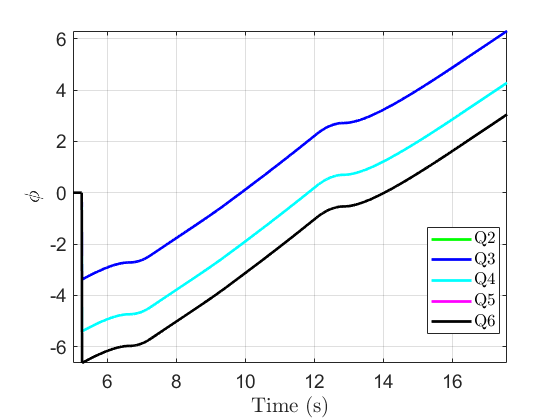}
    \caption{Six Vehicle Test $\phi$}
    \label{fig:six_vehicle_test_phi}
\end{figure}

\begin{figure}
    \centering
    \includegraphics[width=\columnwidth,trim=0 0 0 0, clip]{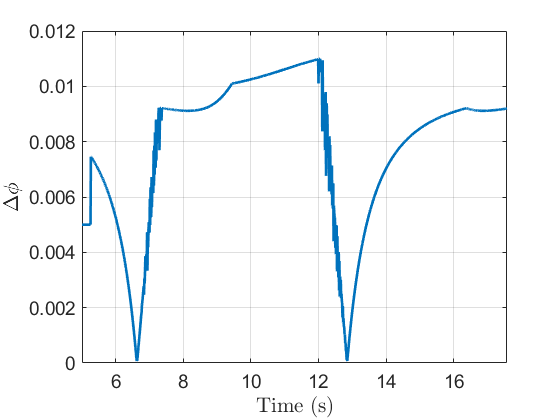}
    \caption{Six Vehicle Test $\Delta \phi$}
    \label{fig:six_vehicle_test_dphi}
\end{figure}

\begin{figure}
    \centering
    \includegraphics[width=\columnwidth,trim=0 0 0 0, clip]{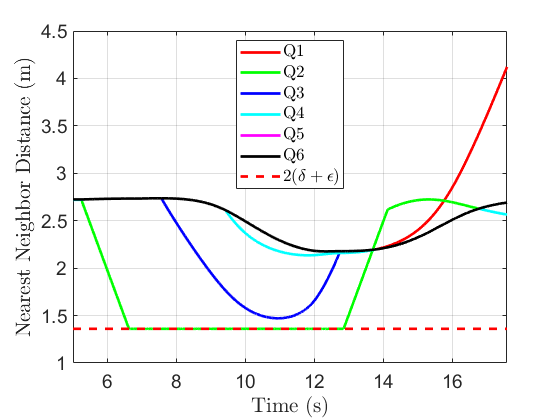}
    \caption{Six Vehicle Test Nearest Neighbor Distance}
    \label{fig:six_vehicle_test_nn}
\end{figure}

To explore the effectiveness of the velocity tracking aspect of our navigation algorithm the six vehicle flight was repeated while the value for K was varied in the range \{1,2,5\}. Figure \ref{fig:2022_02_25_K_plots_vel} shows the maximum velocity of any agent for the various values of K during one of the two saddle points that cause the most difficulty. K=1 reached a 2D speed of 2.4 m/s, much higher than the 1.0 m/s setpoint. Speed error with K=2 was significant lower, while K=5 performed a bit better. However, increasing the value of K essentially replans with a new $\Delta \phi$ K times each iteration. This has a runtime cost which can be seen in Figure \ref{fig:2022_02_25_K_plots_runtime}. Runtime is proportional to the value of K and the number of agents. Runtimes of roughly 0.2ms, 0.4ms, and 1ms were measured for the values of 1, 2, and 5 respectively. 10ms is the hard upper limit since the navigation algorithm is running at 100Hz. A value of K=2 follows the velocity profile well while offering a reasonably low runtime.

\begin{figure}
    \centering
    \includegraphics[width=\columnwidth,trim=0 0 0 0, clip]{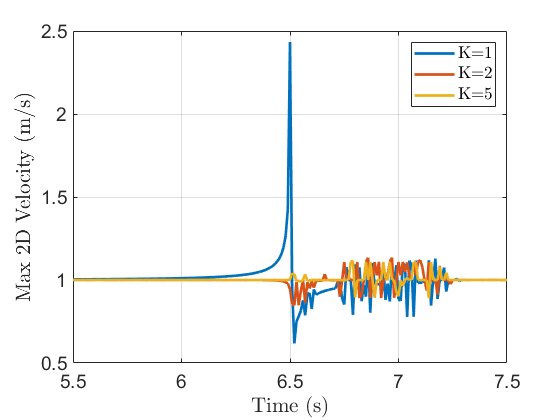}
    \caption{How Varying K affects Velocity Tracking}
    \label{fig:2022_02_25_K_plots_vel}
\end{figure}

\begin{figure}
    \centering
    \includegraphics[width=\columnwidth,trim=0 0 0 0, clip]{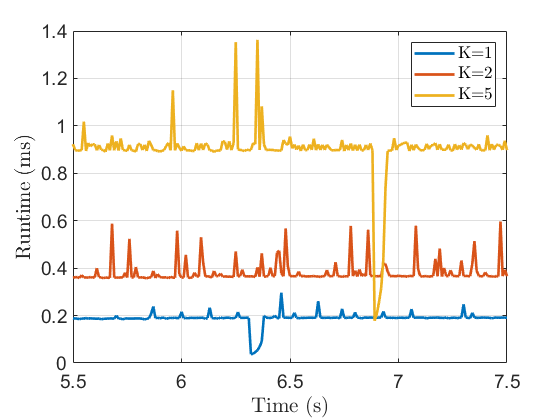}
    \caption{How Varying K affects Runtime}
    \label{fig:2022_02_25_K_plots_runtime}
\end{figure}

\subsection{Two Vehicle Flight Test}
The same formation from the previous section in a reduced configuration (only Q1 and Q2) was flight tested to validate the navigation algorithm working safely to avoid collisions with non-zero controller error bounded by $\delta=40cm$ for each vehicle. Figure \ref{fig:teaser_figure} shows this test.
Figure \ref{fig:2022_02_24_test10_2d_plot} shows the trajectories for failed agent Q1 within its exclusion zone (solid red circle) and Q2 which tracked its desired trajectory computed in real-time (dotted green) closely by its actual trajectory (solid green). The controller error of both agents remained within the $\delta$ bound as seen in Figure \ref{fig:2022_02_24_test10_controller_error}. Figure \ref{fig:2022_02_24_test10_exclusion_zone_distance} shows that the actual trajectory of Q2 was a safe distance away from the exclusion zone. The dynamic value of $\Delta \phi$ can be seen in Figure \ref{fig:2022_02_24_test10_dphi}.

\begin{figure}
    \centering
    \includegraphics[width=\columnwidth,trim=0 0 0 0, clip]{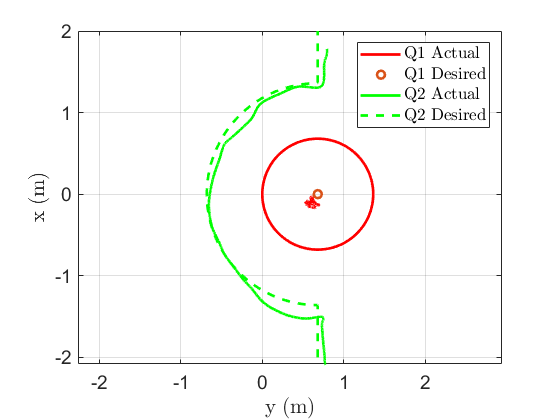}
    \caption{Two Vehicle 2D Plot }
    \label{fig:2022_02_24_test10_2d_plot}
\end{figure}

\begin{figure}
    \centering
    \includegraphics[width=\columnwidth,trim=0 0 0 0, clip]{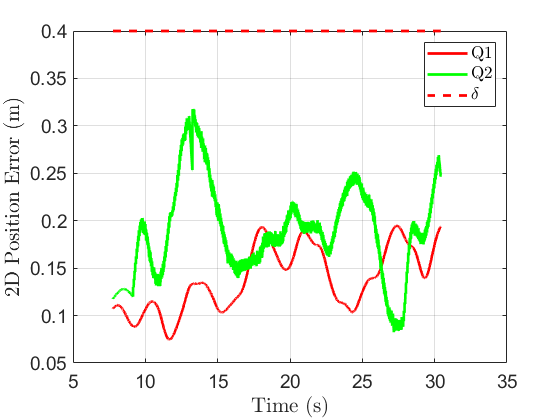}
    \caption{Two Vehicle Controller Error}
    \label{fig:2022_02_24_test10_controller_error}
\end{figure}

\begin{figure}
    \centering
    \includegraphics[width=\columnwidth,trim=0 0 0 0, clip]{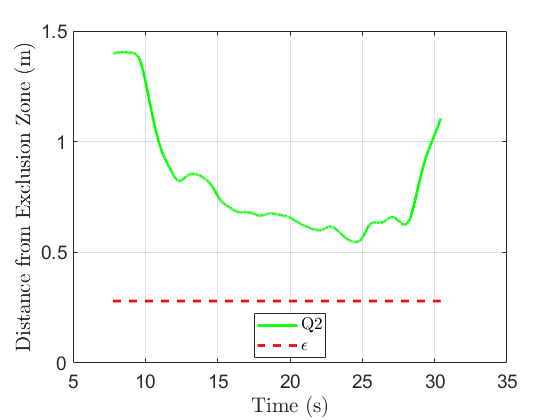}
    \caption{Two Vehicle Exclusion Zone Distance}
    \label{fig:2022_02_24_test10_exclusion_zone_distance}
\end{figure}

\begin{figure}
    \centering
    \includegraphics[width=\columnwidth,trim=0 0 0 0, clip]{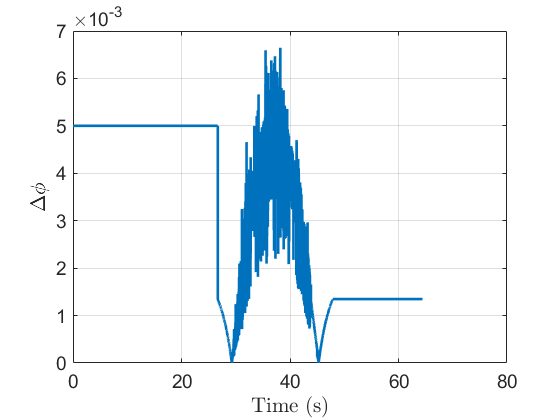}
    \caption{Two Vehicle $\Delta \phi$ }
    \label{fig:2022_02_24_test10_dphi}
\end{figure}

\section{Conclusion}
\label{Sec:Conclusion}
This paper has presented an algorithm for quadrotor formation flight that allows healthy quadrotors to safely continue their mission in the presence of quadrotor failure(s). We presented a novel real-time fluid-flow based CEM navigation algorithm where recovery trajectories of healthy quadrotors are calculated as streamlines of an ideal fluid flow. We analyzed two safety conditions that specify the minimum separation distance required to assure safety. Experimental results validated the safety condition provided.  

Experiments were conducted with quadrotors flying in the $x$ direction of the $x$-$y$ plane i.e., constant $\psi$ direction of the $\phi$-$\psi$ plane. Using a rotation matrix, the CEM navigation algorithm can be improved further to accommodate quadrotor formation flight in any direction of the $x$-$y$ plane. Although this paper demonstrated experiments with a two quadrotor team and a single failure as well as a six virtual quadrotor team with a single failure, using Theorem \ref{thm1}, the CEM navigation algorithm \ref{CEM Planner (Algorithm)} can be extended to  large quadrotor groups and multiple quadrotor failures. We will extend simulation and experiments to larger teams with multiple failed quadrotors in future work. 

\section*{Acknowledgements}
This work was supported in part by the National Science Foundation (NSF)  under Award Nos. 2133690, 1914581, and 1739525.

% Break to a new page.  The above must all fit on 6 pages.
%\clearpage
\bibliographystyle{IEEEtran}
\bibliography{IEEEabrv,refs}

\begin{thebibliography}{10}
\providecommand{\url}[1]{#1}
\csname url@rmstyle\endcsname
\providecommand{\newblock}{\relax}
\providecommand{\bibinfo}[2]{#2}
\providecommand\BIBentrySTDinterwordspacing{\spaceskip=0pt\relax}
\providecommand\BIBentryALTinterwordstretchfactor{4}
\providecommand\BIBentryALTinterwordspacing{\spaceskip=\fontdimen2\font plus
\BIBentryALTinterwordstretchfactor\fontdimen3\font minus
  \fontdimen4\font\relax}
\providecommand\BIBforeignlanguage[2]{{%
\expandafter\ifx\csname l@#1\endcsname\relax
\typeout{** WARNING: IEEEtran.bst: No hyphenation pattern has been}%
\typeout{** loaded for the language `#1'. Using the pattern for}%
\typeout{** the default language instead.}%
\else
\language=\csname l@#1\endcsname
\fi
#2}}

\bibitem{du2017pursuing}
S.-L. Du, X.-M. Sun, M.~Cao, and W.~Wang, ``Pursuing an evader through
  cooperative relaying in multi-agent surveillance networks,''
  \emph{Automatica}, vol.~83, pp. 155--161, 2017.

\bibitem{da2017multi}
L.~C.~B. Da~Silva, R.~M. Bernardo, H.~A. De~Oliveira, and P.~F. Rosa,
  ``Multi-uav agent-based coordination for persistent surveillance with dynamic
  priorities,'' in \emph{2017 International Conference on Military Technologies
  (ICMT)}.\hskip 1em plus 0.5em minus 0.4em\relax IEEE, 2017, pp. 765--771.

\bibitem{rojo2017implementation}
E.~Rojo-Rodriguez, E.~Ollervides, J.~Rodriguez, E.~Espinoza,
  P.~Zambrano-Robledo, and O.~Garcia, ``Implementation of a super twisting
  controller for distributed formation flight of multi-agent systems based on
  consensus algorithms,'' in \emph{2017 International Conference on Unmanned
  Aircraft Systems (ICUAS)}.\hskip 1em plus 0.5em minus 0.4em\relax IEEE, 2017,
  pp. 1101--1107.

\bibitem{ha2017multi}
S.-h. Ha and S.-d. Chi, ``Multi-agent based design of autonomous uavs for both
  flocking and formation flight,'' \emph{Journal of Advanced Navigation
  Technology}, vol.~21, no.~6, pp. 521--528, 2017.

\bibitem{altair2017decision}
H.~AlTair, T.~Taha, J.~Dias, and M.~Al-Qutayri, ``Decision making for
  multi-objective multi-agent search and rescue missions,'' in \emph{2017
  International Conference on Electrical and Computing Technologies and
  Applications (ICECTA)}.\hskip 1em plus 0.5em minus 0.4em\relax IEEE, 2017,
  pp. 1--4.

\bibitem{witczuk2018exploring}
J.~Witczuk, S.~Pagacz, A.~Zmarz, and M.~Cypel, ``Exploring the feasibility of
  unmanned aerial vehicles and thermal imaging for ungulate surveys in
  forests-preliminary results,'' \emph{International Journal of Remote
  Sensing}, vol.~39, no. 15-16, pp. 5504--5521, 2018.

\bibitem{tsouros2019review}
D.~C. Tsouros, S.~Bibi, and P.~G. Sarigiannidis, ``A review on uav-based
  applications for precision agriculture,'' \emph{Information}, vol.~10,
  no.~11, p. 349, 2019.

\bibitem{rastgoftar2018cooperative1}
H.~Rastgoftar and E.~M. Atkins, ``Cooperative aerial lift and manipulation
  (calm),'' \emph{Aerospace Science and Technology}, 2018.

\bibitem{rossomando2020aerial}
F.~Rossomando, C.~Rosales, J.~Gimenez, L.~Salinas, C.~Soria,
  M.~Sarcinelli-Filho, and R.~Carelli, ``Aerial load transportation with
  multiple quadrotors based on a kinematic controller and a neural smc dynamic
  compensation,'' \emph{Journal of Intelligent \& Robotic Systems}, vol. 100,
  no.~2, pp. 519--530, 2020.

\bibitem{argrow2005uav}
B.~Argrow, D.~Lawrence, and E.~Rasmussen, ``Uav systems for sensor dispersal,
  telemetry, and visualization in hazardous environments,'' in \emph{43rd AIAA
  Aerospace Sciences Meeting and Exhibit}, 2005, p. 1237.

\bibitem{lewis1997high}
M.~A. Lewis and K.-H. Tan, ``High precision formation control of mobile robots
  using virtual structures,'' \emph{Autonomous robots}, vol.~4, no.~4, pp.
  387--403, 1997.

\bibitem{ren2004decentralized}
W.~Ren and R.~W. Beard, ``Decentralized scheme for spacecraft formation flying
  via the virtual structure approach,'' \emph{Journal of Guidance, Control, and
  Dynamics}, vol.~27, no.~1, pp. 73--82, 2004.

\bibitem{rao2013sliding}
S.~Rao and D.~Ghose, ``Sliding mode control-based autopilots for leaderless
  consensus of unmanned aerial vehicles,'' \emph{IEEE transactions on control
  systems technology}, vol.~22, no.~5, pp. 1964--1972, 2013.

\bibitem{ren2009distributed}
W.~Ren, ``Distributed leaderless consensus algorithms for networked
  euler--lagrange systems,'' \emph{International Journal of Control}, vol.~82,
  no.~11, pp. 2137--2149, 2009.

\bibitem{song2010second}
Q.~Song, J.~Cao, and W.~Yu, ``Second-order leader-following consensus of
  nonlinear multi-agent systems via pinning control,'' \emph{Systems \& Control
  Letters}, vol.~59, no.~9, pp. 553--562, 2010.

\bibitem{xu2016robust}
L.~Xu, H.~Fan, Z.~Dong, and W.~Wang, ``Robust consensus control for
  leader-following multi-agent system under switching topologies,'' in
  \emph{2016 International Conference on Cybernetics, Robotics and Control
  (CRC)}.\hskip 1em plus 0.5em minus 0.4em\relax IEEE, 2016, pp. 27--31.

\bibitem{ferrari2006laplacian}
G.~Ferrari-Trecate, M.~Egerstedt, A.~Buffa, and M.~Ji, ``Laplacian sheep: A
  hybrid, stop-go policy for leader-based containment control,'' in
  \emph{International Workshop on Hybrid Systems: Computation and
  Control}.\hskip 1em plus 0.5em minus 0.4em\relax Springer, 2006, pp.
  212--226.

\bibitem{li2013distributed}
Z.~Li, W.~Ren, X.~Liu, and M.~Fu, ``Distributed containment control of
  multi-agent systems with general linear dynamics in the presence of multiple
  leaders,'' \emph{International Journal of Robust and Nonlinear Control},
  vol.~23, no.~5, pp. 534--547, 2013.

\bibitem{wang2013distributed}
X.~Wang, S.~Li, and P.~Shi, ``Distributed finite-time containment control for
  double-integrator multiagent systems,'' \emph{IEEE Transactions on
  Cybernetics}, vol.~44, no.~9, pp. 1518--1528, 2013.

\bibitem{wen2015containment}
G.~Wen, Y.~Zhao, Z.~Duan, W.~Yu, and G.~Chen, ``Containment of higher-order
  multi-leader multi-agent systems: A dynamic output approach,'' \emph{IEEE
  Transactions on Automatic Control}, vol.~61, no.~4, pp. 1135--1140, 2015.

\bibitem{rastgoftar2014evolution}
H.~Rastgoftar and S.~Jayasuriya, ``Evolution of multi-agent systems as
  continua,'' \emph{J. Dynamic Sys., Meas., \& Control}, vol. 136, no.~4, 2014.

\bibitem{rastgoftar2016continuum}
H.~Rastgoftar, \emph{Continuum deformation of multi-agent systems}.\hskip 1em
  plus 0.5em minus 0.4em\relax Springer.

\bibitem{rastgoftar2017multi}
H.~Rastgoftar and E.~M. Atkins, ``Multi-quadcopter system continuum deformation
  optimization,'' \emph{IEEE Transactions on Automatic Control (submitted)},
  2017.

\bibitem{emadi2021physics}
H.~Emadi, H.~Uppaluru, and H.~Rastgoftar, ``A physics-based safety recovery
  approach for fault-resilient multi-quadcopter coordination,'' \emph{arXiv
  preprint arXiv:2110.07777}, 2021.

\bibitem{khatib1985real}
O.~Khatib, ``Real-time obstacle avoidance for manipulators and mobile robots,''
  in \emph{Proceedings. 1985 IEEE International Conference on Robotics and
  Automation}, vol.~2.\hskip 1em plus 0.5em minus 0.4em\relax IEEE, 1985, pp.
  500--505.

\bibitem{barraquand1991robot}
J.~Barraquand and J.-C. Latombe, ``Robot motion planning: A distributed
  representation approach,'' \emph{The International Journal of Robotics
  Research}, vol.~10, no.~6, pp. 628--649, 1991.

\bibitem{connolly1990path}
C.~I. Connolly, J.~B. Burns, and R.~Weiss, ``Path planning using laplace's
  equation,'' in \emph{Proceedings., IEEE International Conference on Robotics
  and Automation}.\hskip 1em plus 0.5em minus 0.4em\relax IEEE, 1990, pp.
  2102--2106.

\bibitem{romano2019experimental}
M.~Romano, P.~Kuevor, D.~Lukacs, O.~Marshall, M.~Stevens, H.~Rastgoftar,
  J.~Cutler, and E.~Atkins, ``Experimental evaluation of continuum deformation
  with a five quadrotor team,'' in \emph{2019 American Control Conference
  (ACC)}.\hskip 1em plus 0.5em minus 0.4em\relax IEEE, 2019, pp. 2023--2029.

\end{thebibliography}

\end{document}